\documentclass{article}
\usepackage{graphicx}
\usepackage{amsmath,bm,bbm}
\usepackage{enumerate}
\usepackage{amssymb}
\usepackage{enumitem}
\usepackage{mathrsfs}
\usepackage{mathtools}
\usepackage{cite}
\usepackage{amsthm}
\usepackage{algorithm}
\usepackage{algorithmicx}
\usepackage[dvipsnames]{xcolor}
\usepackage{algpseudocode}
\usepackage[colorlinks,citecolor=blue,urlcolor=blue]{hyperref}
\usepackage{hyperref,cleveref}
\usepackage{geometry}
\usepackage{authblk}
\usepackage{pdfpages}
\usepackage{mdframed}
\usepackage{eqparbox}
\usepackage[all]{xy}
\usepackage[table]{xcolor}

\DeclareMathOperator*{\argmax}{arg\,max}

\newtheorem{lemma}{Lemma}
\newtheorem{theorem}{Theorem}

\newtheorem{definition}{Definition}
\newtheorem*{remark}{Remark}
\newtheorem{example}{Example}

\newtheorem{question}{Question}

\newcommand\naturalnumber{\mathbb{N}}

\newcommand\E{\mathbb{E}}

\newcommand\rwc{\mathcal{R}}
\newcommand\piref{\pi_{\mathrm{ref}}}

\newcommand{\longsim}{\mathrel{\scalebox{3.0}[1.0]{$\sim$}}}

\def\ddefloop#1{\ifx\ddefloop#1\else\ddef{#1}\expandafter\ddefloop\fi}
\def\ddef#1{\expandafter\def\csname v#1\endcsname{\ensuremath{\boldsymbol{#1}}}}
\ddefloop abcdefghijklmnopqrstuvwxyzABCDEFGHIJKLMNOPQRSTUVWXYZ\ddefloop
\def\ddef#1{\expandafter\def\csname v#1\endcsname{\ensuremath{\boldsymbol{\csname #1\endcsname}}}}
\ddefloop {alpha}{beta}{gamma}{delta}{epsilon}{varepsilon}{zeta}{eta}{theta}{var
theta}{iota}{kappa}{lambda}{mu}{nu}{xi}{pi}{varpi}{rho}{varrho}{sigma}{varsigma}
{tau}{upsilon}{phi}{varphi}{chi}{psi}{omega}{Gamma}{Delta}{Theta}{Lambda}{Xi}{Pi
}{Sigma}{varSigma}{Upsilon}{Phi}{Psi}{Omega}{ell}\ddefloop
	
\newlength\oversetwidth
\newlength\underwidth

\title{Towards a theory of inference-time alignment with \\ unknown rewards}

\author{
Steve Hanneke\thanks{steve.hanneke@gmail.com}
\qquad
Hongao Wang\thanks{wang5270@purdue.edu}
\qquad
Mingyue Xu\thanks{xu1864@purdue.edu}
\\
\normalsize {\sl Purdue University}
}

\date{}

\begin{document}
\maketitle

\begin{abstract}
    Generative model alignment has received broad interest, and significant progress has been made in supervised fine-tuning and inference-time computation. 
    Yet, alignment has remained poorly understood from a statistical learning perspective. 
    We formulate inference-time alignment as a weak-to-strong learning problem, where a reference policy (weak model) is assumed to be fairly good and the goal is to produce a strong model that predicts a good response at test time with arbitrarily high probability.
    Our problem is formulated as learning from scratch --- everything is learned from data rather than assuming access to a good reward estimate, and thus differs from the existing inference-time alignment theory.
    Our framework shares similarity to the recent work of Joshi et al., \cite{Joshi2026learning}, where for each prompt, there could be multiple good responses.
    Our definition of the alignment learnability follows the standard PAC learning principle.
    We introduce a novel combinatorial dimension of the reward class which we call the alignment dimension, and show that it completely characterizes the alignment learnability --- a reward class is alignment learnable if and only if its alignment dimension is finite.
    The core of our learning procedure works by learning a pairwise comparator and then running a tournament over candidate responses. 
    We believe that our results might shed light toward establishing a complete theoretical understanding of alignment.
\end{abstract}

\section{Introduction}
Alignment --- aligning AI agents to behave in accordance with human intentions --- is motivated by the observation that while modern generative models acquire impressive capabilities from scalable training, they may exhibit unintended or harmful behavior \cite{Dario2016concrete,Ouyang2022training}. 
Let us take the language model alignment as example.\footnote{Throughout the paper, we will take the language model alignment as example for presentation, it is worth emphasizing that our theory captures inference-time alignment in general for other generative models.}
Traditional methodology for language model alignment falls into two categories, namely, Reinforcement Learning from Human/AI Feedback (RLHF/RLAIF) \cite{Christiano2017deep,ziegler2020finetuninglanguagemodelshuman,bai2022constitutionalaiharmlessnessai,lee2024rlaifvsrlhf} and inference-time alignment \cite{Khanov2024args,sun2024fast,huang2025bestofn}. 
RLHF/RLAIF methods usually involve two steps, where a parameterized reward function is first learned using the curated preference data annotated by human beings or AI, and is then used to optimize the model to produce responses that receive higher rewards.
In contrast, inference-time alignment approaches work by selecting or steering the model's generated responses according to a reward signal at inference time, and is thus free from updating the model parameters. 
We defer a detailed discussion of these two paradigms to \Cref{sec:related-works}.

Yet, despite its increasing popularity in both empirical and theoretical investigations, only a very few works have tried to model the alignment problem from a statistical learning perspective and explored its learning performance guarantee.
In fact, theoretical works on RLHF have largely focused on designing better preference-optimization objectives, motivated by the empirical observation of the reward overoptimization (a.k.a. reward hacking) \cite{gheshlaghi-azar2024general,liu2024provably,huang2025correcting,xie2025exploratory};
or on deriving sample complexity guarantees under structured learning models, such as assuming linear rewards, log-linear policies, and contextual-bandit formulation \cite{xiong2024iterative,nika2024reward,li2024policy,das2026active}.
For inference-time alignment, existing theoretical researches have conditioned on a reward model satisfying prescribed accuracy or uncertainty assumptions \cite{huang2025bestofn,beirami2025theoretical,aminian2026best,sriraman2026revisiting}, rather than studying the end-to-end statistical problem of learning the reward model from finite preference data and subsequently using it to align unseen test instances.
Motivated by this, we ask a more basic statistical learning question: 

\begin{question}
  \label{ques:learnability-question}
    Which reward classes are learnable for inference-time alignment, given access only to the preference data?
\end{question}

In this work, we provide an answer to the above question by studying the problem of inference-time alignment from a statistical learning perspective. 
For simplicity, we call our problem alignment rather inference-time alignment whose formal definition can be found in \Cref{sec:preliminaries}.
As a preliminary work, we restrict to the basic setting where the reward class contains binary reward functions, i.e., $\rwc\subseteq\{0,1\}^{\mathcal{X}\times\mathcal{Y}}$ and assume that the true reward function lies in class, i.e., $r^{*}\in\rwc$.
Unlike the existing routes, we formulate the alignment problem as a boosting/weak-to-strong learning problem, where the data is generated from a reference/base policy $\piref$ whose performance is assumed to be fair, in a sense that it has a constant probability of generating a good response (with $r^*$ value 1) for a given input prompt. 
Therefore, a reference policy can be thought of as a weak model in the context of weak-to-strong learning. 
In practice, $\piref$ is usually obtained from a pretrained language model or a supervised fine-tuning (SFT) model.
For instance, InstructGPT uses the pretrained GPT-3 as the reference model \cite{Ouyang2022training}.\footnote{Other common reference models include the LLaMA series \cite{touvron2023llamaopenefficientfoundation}, the Gemma series \cite{gemmateam2024gemmaopenmodelsbased}, the Qwen series \cite{bai2023qwentechnicalreport}, etc.} 
This is because pretrained models typically possess nice properties such as linguistic fluency and broad knowledge, but their pretraining does not pay much attention to following human instructions or preferences. 
On this account, practical language model alignment methods are usually designed to output selective responses generated from (some policy close to) $\piref$.
The goal of alignment is to boost the performance of the reference model using the training data with human-annotated preferences. 
We adopt the classical $(\epsilon,\delta)$ standard of statistical learning to measure the performance, namely, with probability at least $1-\delta$, the probability that the learner (aligned model) ``predicts'' a response with $r^*$ value 0 is at most $\epsilon$. 
Thus, the aligned model can be viewed as a stronger model as the target of boosting.
In particular, we consider inference-time alignment --- the learner is expected to select a response with $r^*$ value 1 from multiple $\piref$-generated candidates with arbitrarily high probability at test time.

Our main contributions are summarized as follows.
To our best knowledge, we are the first to develop a statistical learning framework for the general alignment problem, in particular inference-time alignment. 
Our problem is formulated as learning from scratch --- everything is learned from data, which we call ``inference-time alignment with unknown rewards''.
This distinguishes our work from existing inference-time alignment theoretical works since they assume having access to a good reward estimate (see \Cref{sec:related-works} for a detailed discussion).
Although being superficially related to several existing statistical learning frameworks such as 
(i) multiclass learning --- the learner outputs only one of the multiple labels for each test instance \cite{Natarajan1988two,natarajan1989learning} 
and (ii) list learning --- the learner outputs multiple plausible labels for each instance rather than just one \cite{Brukhim2022characterization}, 
we notice that our formulation differs from them --- multiple labels might be correct but the learner only needs to find one of them.
In fact, our learning framework is more close to the recent works of \cite{Joshi2026learning,pour2026learningmultiplecorrectanswers} (see \Cref{sec:related-works} for a detailed discussion), and might be of independent interest to the learning theory community. 
We propose a novel combinatorial dimension of the reward class, which we name the alignment dimension.
We then provide a complete characterization of the learnability of the alignment problem, which we call the alignment learnability --- any reward class is alignment learnable if and only if its alignment dimension is finite. 
We consider our work as an initial step towards establishing a statistical learning theory for the general alignment problem, potentially shedding light on developing more systematic principles for alignment in practice.

\section{Preliminaries}
  \label{sec:preliminaries}
Throughout the paper, let $\mathcal{X}$ be an instance space and $\mathcal{Y}$ be an action space. Let $\rwc\subseteq\{0,1\}^{\mathcal{X}\times\mathcal{Y}}$ be a reward class containing reward functions $r:\mathcal{X}\times\mathcal{Y}\rightarrow\{0,1\}$. Fix some $k\in\naturalnumber$. Let $\mathcal{Z}_{k}=\mathcal{X}\times\mathcal{Y}^{k}$. For any reward function $r:\mathcal{X}\times\mathcal{Y}\mapsto\{0,1\}$, define its induced set-valued acceptance function as $A_{r}:\mathcal{Z}_{k}\mapsto 2^{[k]}$ satisfying $A_{r}(z) = \{j\in[k]: r(x,y_{j})=1\}$ for any $z\in\mathcal{Z}_{k}$. For a reward class $\rwc$, define its induced set-valued acceptance class as $\mathcal{A}_{k}(\rwc)=\{A_{r}: r\in\rwc\}$. Our learning problem can be described as follows.

Assume the realizable case where an unknown target reward function $r^{*}\in\rwc$ is to be learned. Let $\mathcal{D}_{X}$ be some unknown distribution over $\mathcal{X}$ and let $\gamma\in(0,1)$. Let $\piref$ be a conditional distributional measure such that for $\mathcal{D}_{X}$-almost every $x\in\mathcal{X}$,
\begin{equation*}
    \piref(y\in\mathcal{Y}: r^{*}(x,y)=1|x) \geq \gamma .
\end{equation*}
It is worthwhile to emphasize that we only have oracle (query) access to $\piref$, but have no distribution information of it.
For a sample size $m\in\naturalnumber$, we collect $m$ i.i.d.\ samples,
\begin{equation*}
    S_{m}=\{(Z_{1},A_{r^{*}}(Z_{1})),\ldots,(Z_{m},A_{r^{*}}(Z_{m}))\} ,
\end{equation*}
according to the following. For every $i\in[m]$, $Z_{i}=(X_{i},Y_{i}^{(1)},\ldots,Y_{i}^{(k)})\in\mathcal{Z}_{k}$ is a random tuple where 
\begin{equation}
  \label{eq:tuple-distribution}
    X_{i}\sim\mathcal{D}_{X} \;\;\text{and}\;\; Y_{i}^{(1)},\ldots,Y_{i}^{(k)}\overset{\text{i.i.d.}}{\longsim}\piref(\cdot|X_{i}) .
\end{equation}
In brief, each sample contains an instance $X_{i}$, a size-$k$ slate of labels $Y_{i}^{(1)},\ldots,Y_{i}^{(k)}$, and an acceptance set $A_{r^{*}}(Z_{i})\subseteq[k]$ indicating which labels are accepted by the target reward $r^{*}$. For simplicity, we denote by $\mathcal{D}_{k}$ the distribution over $\mathcal{Z}_{k}$ defined in \Cref{eq:tuple-distribution} for a fixed $k\in\naturalnumber$. At test time, we are given a test instance $X\sim\mathcal{D}_{X}$ accompanied with a size-$k$ slate of labels $Y^{(1)},\ldots,Y^{(k)}$ i.i.d.\ generated from $\piref(\cdot|X)$. 

The goal is to design a learning algorithm $\mathcal{A}$ that takes $S_{m}$ as input and outputs a hypothesis $\mathcal{A}(S_{m}):\mathcal{Z}_{k}\rightarrow[k]$ that, given a test tuple $Z=(X,Y^{(1)},\ldots,Y^{(k)})$, predicts an index $\mathcal{A}(S_{m})(Z)=\hat{i}_{m}(Z)\in[k]$ such that the corresponding response $Y^{(\hat{i}_{m}(Z))}$ satisfying $r^{*}(X,Y^{(\hat{i}_{m}(Z))})=1$ with high probability. 
Note that this exactly matches the inference-time alignment since the learner selects a good response from candidates rather than predicts one.
More formally, we define the learnability of our problem as follows.

\begin{definition}  [\textbf{Alignment learnability}]
  \label{def:alignment-learnability}
A reward class $\rwc$ is called \underline{alignment learnable} if there exists a learning algorithm $\mathcal{A}$ such that the following holds. For any $\gamma,\epsilon,\delta\in(0,1)$, there exists $k(\gamma,\epsilon,\delta)<\infty$ such that, for any $k\geq k(\gamma,\epsilon,\delta)$, there exists $m(k,\epsilon,\delta)<\infty$ such that, for every target reward function $r^{*}\in\rwc$, every distribution $\mathcal{D}_{X}$ over $\mathcal{X}$, every reference policy $\piref$ satisfying
\begin{equation*}
  \Pr_{Y\sim\piref(\cdot|x)}\left(r^{*}(x,Y)=1\right)\geq\gamma
\end{equation*}
for $\mathcal{D}_{X}$-almost every $x$, and every $m\geq m(k,\epsilon,\delta)$,
\begin{equation*}
    \Pr_{Z\sim\mathcal{D}_{k}}\left[\hat{i}_{m}(Z)\in A_{r^{*}}(Z)\right] \geq 1-\epsilon ,
\end{equation*}
with probability at least $1-\delta$ over the randomness of $S_{m}$ and the internal randomness of the learner.
\end{definition}

\begin{remark}
In conclusion, a reward class $\rwc$ is alignment learnable if it requires (i) a finite training sample size $m$; and (ii) a finite inference-time sampling round $k$.
The coverage assumption implies that
\begin{equation*}
  k\geq \frac{1}{\gamma}\log\left(\frac{1}{\epsilon}\right)
\end{equation*}
suffices to guarantee that a good response exists in the test slate with probability at least $1-\epsilon$.
The finiteness of the alignment dimension further ensures the finiteness of $m$.
\end{remark}

\subsection{Main results}
  \label{subsec:learnability-characterization}
Our main result is a clean characterization of the alignment learnability as defined in \Cref{def:alignment-learnability}, thus completely answering Question~\ref{ques:learnability-question}.
We introduce a novel combinatorial dimension, which we call the ``alignment dimension'', and show that the alignment dimension of a reward class characterizes its alignment learnability (\Cref{thm:alignment-learnability}).

\begin{definition}  [\textbf{Alignment dimension}]
  \label{def:alignment-dimension}
Let $d\in\naturalnumber$. We say that a set of $d$ triples $\{z_1, \ldots,z_d\}$, where for every $i\in[d]$,
\begin{equation*}
  z_i=(x_i,y_i^{(0)},y_i^{(1)})\in\mathcal{X}\times\mathcal{Y}\times\mathcal{Y} 
\end{equation*}
are \underline{alignment shattered by $\rwc$} if, for every boolean vector $\vb\in\{0,1\}^{d}$, there exists a reward function $r_{\vb}\in\rwc$ satisfying that for every $i\in[d]$,
\begin{equation*}
    r_{\vb}\left(x_i,y_i^{(b_i)}\right)=1
    \;\;\mathrm{and}\;\;
    r_{\vb}\left(x_i,y_i^{(1-b_i)}\right)=0.
\end{equation*}
It is worth emphasizing that $x_1,\ldots,x_d$ are not required to be distinct. The \underline{alignment dimension of $\rwc$}, denoted by $d_{\mathrm{ALN}}(\rwc)$, is the largest $d\in\naturalnumber\cup\{0\}$ for which such an alignment-shattered set exists. If arbitrarily large finite sets are alignment shattered, we say that $d_{\mathrm{ALN}}(\rwc)=\infty$.
\end{definition}

\begin{theorem}
  \label{thm:alignment-learnability}
$\rwc$ is alignment learnable if and only if $d_{\mathrm{ALN}}(\rwc)<\infty$.
\end{theorem}

\section{Related works}
  \label{sec:related-works} 
\paragraph{RLHF/RLAIF.} 
Reinforcement Learning from Human/AI Feedback (RLHF/RLAIF) framework is summarized as follows. 
Given an input/prompt $x$ and an action/response $y$, the evaluation of $y$ on $x$ is assumed to be generated by some latent reward model $r^{*}(x,y)$, which lies in a reward class $\rwc$. 
For a pair of responses $(y^{w},y^{l})$ on $x$, human preferences are captured by a theoretical preference model called the Bradley-Terry model \cite{Bradley1952rank}. 
Specifically, given $r^{*}$ being the true reward function, the probability that a response $y^{w}$ is preferred over another response $y^{l}$ on the prompt $x$ is stipulated as
\begin{equation*}
    \mathbb{P}_{r^{*}}\left(y^{w}\succ y^{l}\big| x\right) = \frac{\exp{\{r^{*}(x,y^{w})\}}}{\exp{\{r^{*}(x,y^{w})\}}+\exp{\{r^{*}(x,y^{l})\}}} =: \sigma\left(r^{*}(x,y^{w})-r^{*}(x,y^{l})\right) ,
\end{equation*}
where $\sigma(x)=1/(1+e^{-x})$ is the logistic function.
RLHF/RLAIF methods usually consist of two steps --- reward maximization followed by policy searching via KL-divergence regularized reinforcement learning.
The training preference dataset $\mathcal{S}_{n}:=\{(x_{i},y^{w}_{i},y^{l}_{i})\}_{i=1}^{n}$ is generated according to $y^{w},y^{l}\sim\piref(\cdot|x)$ where $\piref$ is called the reference policy, and the preferences $y^{w} \succ y^{l}$ are annotated by human beings/AI. 
Reward maximization refers to searching for a reward function $r\in\rwc$ to maximize the following likelihood objective
\begin{equation*}
    \mathcal{L}\left(r,\mathcal{S}_{n}\right) = \E_{(x,y^{w},y^{l})\sim\mathcal{S}_{n}}\left[\log\sigma\left(r(x,y^{w})-r(x,y^{l})\right)\right] .
\end{equation*}
The learned reward function $\hat{r}$ will be used in the next step of RL fine-tuning --- searching for a policy (conditional distribution) $\pi$ from some policy hypothesis class $\Pi$ to maximize the following KL-divergence regularized RL objective
\begin{equation*}
    \max_{\pi\in\Pi}\left\{\E_{x\sim\mathcal{S}_{n},y\sim\pi(\cdot|x)}\left[\hat{r}(x,y)\right] - \beta\cdot\mathbb{D}_{\text{KL}}\left(\pi(y|x) \;||\; \pi_{\text{ref}}(y|x)\right)\right\} ,
\end{equation*}
where $\beta>0$ is a hyperparameter to control the weight of the KL-regularizer. 

This objective of RLHF/RLAIF has been extensively studied in the past few years \cite{Christiano2017deep,ziegler2020finetuninglanguagemodelshuman,stiennon2020learning,bai2022constitutionalaiharmlessnessai,bai2022traininghelpfulharmlessassistant,lee2024rlaifvsrlhf}. 
Early works mainly focused on designing efficient RL algorithms such as the Proximal Policy Optimization (PPO) \cite{schulman2017proximal} to learn a policy based on a well-estimated reward function. 
Motivated by the fact that the traditional RLHF pipeline is complex, computationally heavyweight and often unstable, \cite{Rafailov2023direct} proposed the celebrated Direct Preference Optimization (DPO) algorithm, which directly optimize a policy to adhere to human preferences in a single stage of policy training, without explicit reward modeling.
Since then, substantial research effort has shifted toward developing such direct optimization algorithms for various preference learning settings such as iterative/on-policy preference learning \cite{yuan2024self,xiong2024iterative,wu2025selfplay}, active preference learning \cite{muldrew2024active,das2026active}, online/exploratory preference learning \cite{xie2025exploratory}, and preference learning without using reference model \cite{meng2024simpo}. 
More recently, motivated by the fact that many reasoning tasks admit rewards that can be automatically checked by programmatic verifiers rather than relying on expensive human-provided preference signals (e.g., verifying the correctness on mathematical problems and code execution), Reinforcement Learning with Verifiable Rewards (RLVR) has been proposed as a variant of RLHF \cite{lambert2025tulu,shao2024deepseekmathpushinglimitsmathematical,wen2026reinforcement}. 
Unlike in general RLHF where rewards are typically real-valued, RLVR adopts binary rewards --- assigning $1$ when the generated response passes an automatic verifier and $0$ otherwise --- which exactly matches our formulation. 
Therefore, our formulation captures exactly inference-time alignment with verifiable rewards.

\paragraph{Inference-time alignment.}
Instead of fine-tuning the language model using the learned reward function as in RLHF, an inference-time alignment algorithm keeps the reference policy $\piref$ fixed and uses the reward model only at test time, such as to do rejection sampling \cite{sun2024fast}, to guide decoding token by token \cite{xu2025genarm}, to implement planning and search \cite{yao2023tree}, etc.
For instance, a canonical inference-time alignment algorithm is the Best-of-$N$ (BoN) heuristic \cite{sun2024fast,huang2025bestofn,beirami2025theoretical}. 
Specifically, given an imperfect reward model $\hat{r}$ and a new test input $x$, BoN works by independently sampling $N$ candidate responses $y^{(1)},\ldots,y^{(N)}$ from $\piref(\cdot|x)$, and returns the one with the highest reward of $\hat{r}$, i.e.,
\begin{equation*}
    \hat{y}_{\mathrm{BoN}}(x) = y^{(j)}, \;\;\text{where}\;\; j\in\argmax_{1\leq i\leq N}\hat{r}(x,y^{(i)}) .
\end{equation*}
The imperfect reward model $\hat{r}$ can be an open-source reward model or a reward model trained on data separately.
As not being the main focus, theoretical works of inference-time alignment usually do not involve analyzing the performance of estimating the true reward function $r^{*}$ by $\hat{r}$, but rather quantifying this estimation error and leave it in bounds. 
For instance, the estimation quality of $\hat{r}$ can be measured by the expected squared error with respect to $r^{*}$ under the reference policy $\piref$, i.e.,
\begin{equation*}
    \E_{y\sim\piref(\cdot|x)}\left[\left(\hat{r}(x,y)-r^{*}(x,y)\right)^{2}\right] .
\end{equation*}
In contrast, our learnability result takes both the reward estimation and the test-time computation into consideration since our problem is to learn from scratch --- everything should be learned from data without any additional assumptions.
Therefore, our theory explains the most general version of inference-time alignment problem. 
We refer the readers to the work of \cite{huang2025bestofn} for a detailed review of inference-time alignment researches.

\paragraph{Learning with multiple correct answers.} 
Existing theoretical works that are most relevant to ours are perhaps \cite{Joshi2026learning,pour2026learningmultiplecorrectanswers}. 
They both study the same setting as we consider in this work where there are multiple correct answers for an instance, and the learner only needs to output one of them. 

Similar to us, \cite{Joshi2026learning} studied the problem of language model alignment from a statistical learning perspective. 
They also assumed that the true reward function $r^{*}$ lies in a low-cardinality reward class $\rwc\subseteq\{0,1\}^{\mathcal{X}\times\mathcal{Y}}$.\footnote{They further extend their results to $\rwc\subseteq[0,1]^{\mathcal{X}\times\mathcal{Y}}$. For simplicity, we discuss here, the case of binary rewards.} 
In their primary setting, the learner receives i.i.d.\ samples $\{x_{i},y_{i}\}_{1\leq i\leq m}$ where each sample contains $x_{i}\sim\mathcal{D}$ and $y_{i}$ being one of the good responses to the input $x_{i}$, i.e., $r^{*}(x_{i},y_{i})=1$. 
Indeed, they assumed that each $y_{i}$ is generated from a demonstration policy $\tilde{\pi}(\cdot|x_{i})$ supported exactly on the set of good responses given $x_{i}$.
Denote $V_{r^{*}}(\pi)=\E_{x\sim\mathcal{D},y\sim\pi(\cdot|x)}[r^{*}(x,y)]$.
Their main result is that, with high probability,
\begin{equation*}
    \Big|V_{r^{*}}(\hat{\pi}) - \sup_{\pi}V_{r^{*}}(\pi)\Big| = O\left(\frac{\log{(|\rwc|)}}{m}\right) .
\end{equation*}
Their learning algorithm utilizes the classical multiplicative-weights update by maintaining weights over all candidate reward functions and predicting according to their weighted aggregate.
Moreover, they also studied a natural extension which they called the pass@$k$ metric, where the learner still receives training data with a single correct response per sample, but at the test time, the learned policy can output a tuple of responses whose reward is the best reward among the candidates.
Under this extension, they improved the convergence rate from $O(\log(|\rwc|)/m)$ to $O(\log_{k+1}(|\rwc|)/m)$. 

To clarify our contributions, we list the main difference between our work and theirs as follows. 
(i) They modeled the problem to also include a policy class in order to exactly match the pipeline of RLHF, while our work models the alignment problem without including a policy class mainly for capturing inference-time alignment methods.
(ii) One of the main limitations of their results is that, they restricted to the case where the reward class has a finite cardinality. 
In contrast, our work provides a complete learnability guarantee for any (infinite) reward class.
(iii) They considered that each training sample contains a single correct response, while we assume having multiple responses sampled from the reference policy. 
We leave the learnability of the single-response model for future works (see our discussion in \Cref{sec:discussion}).
(iv) Finally, their pass@$k$ extension is the most close to our formulation. However, the success of a learning algorithm in their model only requires that the output tuple at test time contains a good response (which is very close to the list learning\footnote{List learning assumes exactly one correct label but their model assumes multiple correct labels. In other words, their learner is expected to output a tuple that intersects with the true label set.}), while our definition of success requires to exactly pick out the correct one.

\cite{pour2026learningmultiplecorrectanswers} instead considered an online learning variant.
They introduced three types of reward feedback models, namely, 
the mistake-unknown feedback model --- after predicting, the learner only sees a correct response but is not told whether its prediction was correct or not, 
the mistake-known feedback model --- the learner is told a correct response as well as the correctness of its prediction, 
and the set-valued feedback model --- the learner sees the entire multiset of correct responses after prediction. 
They further introduced three generalized versions of the Littlestone dimension \cite{littlestone1988learning} to characterize the mistake bound of these three online learning frameworks, respectively.

\section{Examples}
  \label{sec:examples}
In this section, we provide several concrete examples that (i) distinguish our alignment problem from the standard binary/multiclass classification problems; (ii) rule out certain candidates of the alignment learnability characterization as well as some natural learning algorithms; and (iii) guide us towards considering the one-inclusion-graph algorithm used in \Cref{sec:proof-learnability}.

\subsection{Alignment is not a standard classification problem}
  \label{subsec:alignment-is-not-standard-classification}
Our first example shows that our alignment problem differs intrinsically from the standard binary/multi-class classification problem, thus ruling out several candidates for the characterization of the alignment learnability.
Before proceeding to the example, we present the formal definition of the classic VC dimension (for binary classification) and the DS dimension (for multiclass classification).

\begin{definition}   [\textbf{VC dimension \cite{VC71}}]
  \label{def:vc-dimension}
We say that a set $\{z_1 = (x_1,y_1),\ldots,z_n = (x_n,y_n)\}$ is shattered by $\rwc$, if for every $\vb\in\{0,1\}^n$, there exists $r_{\vb} \in \rwc$, such that $r_{\vb}(z_i) = b_i$ for every $i\in [n]$. 
    
The \underline{VC dimension of $\rwc$} is defined as the largest $d\in\naturalnumber\cup\{0\}$ such that there exists a set $\{z_1,\ldots, z_d\}$ shattered by $\rwc$. If no largest finite $d$ exists, we say that $\rwc$ has an infinite VC dimension.
\end{definition}

\begin{definition}  [\textbf{DS dimension \cite{daniely2014optimal}}]
  \label{def:ds-dimension}
Let $z_i = (x_i, y_i^{(1)}, \ldots, y_i^{(k)})$. We say that a set $\{z_1,\ldots,z_n\}$ is DS-shattered by $\mathcal{A}_k(\rwc)$, if there exists a finite subset $\mathcal{A}'_k(\rwc)\subseteq \mathcal{A}_k(\rwc)$, for every $A_{r}\in \mathcal{A}'_k(\rwc)$ and for every $i\in [n]$, there exists $A_{r'}\in \mathcal{A}'_k(\rwc)$ such that $A_{r'}(z_i) \neq A_{r}(z_i)$ and for every $j\neq i$, $A_{r'}(z_j) = A_{r}(z_j)$.

The \underline{DS dimension of $\mathcal{A}_k(\rwc)$} is defined as the largest $d\in\naturalnumber\cup\{0\}$ such that there exists a set $\{z_1,\ldots, z_d\}$ DS-shattered by $\mathcal{A}_k(\rwc)$. If no largest finite $d$ exists, we say that $\mathcal{A}_k(\rwc)$ has an infinite DS dimension.
\end{definition}

\begin{example}
 \label{ex:not-vc}
Let $\mathcal{X}$ be an infinite space and $|\mathcal{Y}|=2$. Let $\{z_1 = (x_1,y^{(0)}, y^{(1)}),z_{2} = (x_2,y^{(0)}, y^{(1)}),\ldots\}$ be an infinite space. 
For any $\vb\in\{0,1\}^{\mathbb{N}}$, define a reward function $r_{\vb}$ according to $r_{\vb}(x_i,y^{(0)}) = 1, r_{\vb}(x_i,y^{(1)}) = b_i$, for any $i\in\naturalnumber$. 
Define a reward class $\rwc=\{r_{\vb}, \forall \vb\in\{0,1\}^{\mathbb{N}}\}$.
\end{example}

A natural idea to first learn a reward function by solving a binary classification problem where the fundamental Empirical Risk Minimization (ERM) rule achieves nearly optimal guarantee, and then apply the Best-of-$N$ or rejection sampling at test time according to the learned reward function. 
This idea suggests that the VC dimension of the reward class $\rwc$ might be the characterization of the alignment learnability (achieved by ERM).
However, \Cref{ex:not-vc} refutes this conjecture by providing an instance of $\rwc$ that has an infinite VC dimension, yet is alignment learnable \emph{without any sample}.
To see this, note that 
\begin{itemize}
    \item $\rwc$ has an infinite VC dimension since the (infinite) set $\{(x_1,y^{(1)}), (x_2,y^{(1)}),\ldots\}$ is shattered by $\rwc$.
    \item $\rwc$ has an alignment dimension $d_{\mathrm{ALN}}(\rwc)=0$ since any reward function assigns $1$ to $(x_{i},y^{(0)})$ for any $i\in\naturalnumber$.
    \item $\rwc$ is alignment learnable \emph{without any training samples} using the following learning strategy. 
    For any test prompt $x$, sample $k=\Theta((1/\gamma)\log(1/\epsilon))$ responses from $\piref(\cdot|x)$ at test time. 
    (This ensures that a good response exists with high probability in the candidate pool. See \Cref{subsec:upper-bound} for a rigorous reasoning.) 
    Then, return $y^{(0)}$ if any of the $k$ candidates is $y^{(0)}$, and return $y^{(1)}$ otherwise.
    This is because, if $y^{(0)}$ exists, it is always acceptable since all reward functions (including the target) agree that it is good. 
    Otherwise if $y^{(0)}$ does not exist, since at least one good response is in the pool, $y^{(1)}$ must be a good response. 
    \item The ERM algorithm fails to learn $\rwc$ according to \Cref{def:alignment-learnability}. 
    Let $r^{*}=r_{\mathbf{0}}$ and $\piref(y^{(0)}|x_{i})=\piref(y^{(1)}|x_{i})=1/2$ for any $i\in\naturalnumber$.
    Indeed, for any finite sample size $m$, the worst-case ERM predicts a reward function $\hat{r}_{m}$ that satisfies $\hat{r}_{m}(x_{i},y^{(1)})=1$ for any $x_{i}$ that does not show up in the training sample.
    Therefore, optimizing at test time using $\hat{r}_{m}$ predicts $y^{(1)}$ for a given $x$ and is always wrong.
\end{itemize}

Another natural approach is to consider the alignment problem as learning a set-valued multiclass classification with a finite label space of size $2^{k}$, where each multiclass classifier $A_{r}\in\mathcal{A}_k(\rwc) :\mathcal{Z}_{k} \to 2^{[k]}$. 
This means equivalently, for multiple responses to each prompt, to learn exactly which of them are good. 
It implies that the characterization of the alignment learnability might be the DS dimension of $\mathcal{A}_k(\rwc)$.
However, for the reward class $\rwc$ in \Cref{ex:not-vc}, its induced set-valued acceptance class $\mathcal{A}_2(\rwc)$ has an infinite DS dimension since the (infinite) set $\{z_1, z_2,\ldots \}$ is DS-shattered by $\mathcal{A}_2(\rwc)$. Therefore, the alignment problem also differs in nature from multiclass classification.

\subsection{Learn-and-optimize principle fails}
  \label{subsec:learning-then-optimizting-fails}
Recall that \Cref{ex:not-vc} has ruled out the strategy of learning a reward function by training an ERM on data and then optimizing at test-time based on the learned reward function. 
More generally, it falls into a broad class of algorithms following the ``learn-and-optimize'' principle, that is, to first learn a reward function (not necessarily in the reward class) only based on the training data (independent of the test instance) and then select the best response (under the evaluation of the learned reward function) at test time (e.g., doing Best-of-$N$ or rejection sampling). 
Our next example shows that any learning strategy following the ``learn-and-optimize'' principle fails to guarantee alignment learnability. 

\begin{example}
  \label{ex:learn-and-optimize-fails}
Let $p\in\naturalnumber$. Let $\{z_1 = (x_1,y^{(1)},y^{(2)},\ldots, y^{(2p+1)}), z_2 = (x_2,y^{(1)},y^{(2)},\ldots, y^{(2p+1)}), \ldots\}$ be an infinite space. For any $\vs \in [2p+1]^{\mathbb{N}}$, any $i\in\mathbb{N}$ and any $j\in [2p+1]$, define a reward function $r_{\vs}$ according to 
\begin{equation*}
    r_{\vs}(x_i,y^{(j)}) = 
    \begin{cases}
        1, \quad j \leq s_{i} \\
        0, \quad j > s_{i}
    \end{cases} .
\end{equation*}
Define $\rwc = \{r_{\vs}, \forall \vs \in [2p+1]^{\mathbb{N}}\}$.
\end{example}

Let $\vs = \{s_i\}_{i\in\mathbb{N}}$ such that $s_i$ is i.i.d.\ sampled from $\mathrm{Unif}([2p])$, where $\mathrm{Unif}([2p])$ stands for the uniform distribution on $[2p]$. 
Let $r^* = r_{\vs}$ and $\pi_{\mathrm{ref}}(y^{(s_i)}|x_i) = \pi_{\mathrm{ref}}(y^{(s_i+1)}|x_i) = 1/2$. 
For any test prompt $x_i$ that does not show up in the training data $S_m$, assume that the sampling round $k$ is sufficiently large such that at least one good response exists in the candidate pool.
Now, if the candidate pool contains two distinct responses, for any learned binary reward function $\hat{r}$, $\E_{\vs}[\Pr_{Y\sim\pi_{\mathrm{ref}}(\cdot|x_i)}(\hat{r}(x_i,Y) \neq r^*(x_i,Y))] \geq 1/2$.
Therefore, the learn-and-optimize principle fails here. 
However, $d_{\mathrm{ALN}}(\rwc)=0$ and the alignment learnability is achievable for \Cref{ex:learn-and-optimize-fails} by utilizing test-time sampling (see \Cref{subsec:pessimism-principle-fails} for details). 

We emphasize that \Cref{ex:learn-and-optimize-fails} does not rule out realizing the alignment learnability by the learn-and-optimize principle when a real-valued reward function (which might encode rank information) can be learned for optimization at test time, which we leave open for future work.

\subsection{Pessimism principle fails}
  \label{subsec:pessimism-principle-fails}
Pessimism --- intentionally constructing a conservative, lower-bound estimate of the reward function when uncertainty arises from data scarcity --- has been shown to provide provable benefits in offline RL \cite{jin2021pessimism,xie2021bellman} and reward function learning \cite{zhu2023principled,zhan2024provable}. 

Indeed, we can show that having access to the test instance and a pool of $k=\Theta((1/\gamma)\log(1/\epsilon))$ candidate responses, a pessimism principle realizes the alignment learnability of the reward class defined in \Cref{ex:learn-and-optimize-fails}. 
To see this, note that when the candidate pool contains only one unique response, it must be a good one since $\piref$ is assumed to generate at least one good response within $k$ times of sampling.
When at least two distinct responses are observed, the pessimism principle will only assign reward value $1$ to the response $y^{(j)}$ in the pool that has the smallest index $j\in[2p+1]$, and assign reward value $0$ to all the other responses.
In other words, only the smallest-indexed response, which is guaranteed to be good, will be accepted at test time. 

More formally, a pessimism principle can be interpreted as running the so-called \emph{Closure} algorithm \cite{auer2007new} on a test instance-dependent version space. 
Let $S_{m}=\{Z_{i}, A_{r^{*}}(Z_{i})\}_{i=1}^{m}$ denote the training data. 
We adapt the classical notion of version space \cite{mitchell1977version} to define 
\begin{equation*}
    V_{S_{m}}(\rwc) = \left\{r\in\rwc:\; A_{r}(Z_{i})=A_{r^{*}}(Z_{i}),\;\forall i\in[m]\right\} .
\end{equation*}
Let $z=(x,y^{(1)},\ldots,y^{(k)})$ be a test instance. Define the ``test-aware version space'' as 
\begin{equation*}
    \hat{\rwc}_{m} = \left\{r\in V_{S_{m}}(\rwc):\; \exists i\in[k] \;\;\mathrm{s.t.}\;\; i\in A_{r}(z)\right\} .
\end{equation*}
Now, we attribute the pessimism principle to running the Closure algorithm on $\hat{\rwc}_{m}$, i.e., to return a reward function that computes the pointwise minimum reward value:
\begin{equation*}
    \hat{r}_{\mathrm{closure}}(x,y) = \min\left\{r(x,y): r\in\hat{\rwc}_{m}\right\} ,\;\; \forall (x,y)\in\mathcal{X}\times\mathcal{Y} .
\end{equation*}
And then, using $\hat{r}_{\mathrm{closure}}$ to optimize and select at test time.
Equivalently, the learning procedure selects, for the test instance $z=(x,y^{(1)},\ldots,y^{(k)})$, a response that lies in the intersection of all acceptance sets induced by reward functions in $\hat{\rwc}_{m}$, that is, select $y^{(\hat{i}_{m})}$ with
\begin{equation*}
    \hat{i}_{m} \in [k] \cap \left(\bigcap\nolimits_{r\in\hat{\rwc}_{m}}A_{r}(z)\right) .
\end{equation*}

Our final example, however, suggests that the pessimism principle does not always guarantee the alignment learnability since $\cap_{r\in\hat{\rwc}_{m}}A_{r}(z)=\emptyset$ might happen.

\begin{example}
 \label{ex:pessimism-fails}
Let $\mathcal{X}$ be an infinite space and $|\mathcal{Y}|=2$. Let $\{z_1 = (x_1,y^{(0)}, y^{(1)}),z_{2} = (x_2,y^{(0)}, y^{(1)}),\ldots\}$ be an infinite space. 
For any $\vb\in\{0,1\}^{\mathbb{N}}$ and any $i\in\naturalnumber$, define a reward function $r_{\vb,i}$ according to 
\begin{equation*}
    \begin{cases}
        \begin{aligned}
            &r_{\vb,i}(x_{j}, y^{(0)}) = 1 ,\;\; &&r_{\vb,i}(x_{j}, y^{(1)}) = b_{j}, &&&j\neq i \\
            &r_{\vb,i}(x_{j}, y^{(0)}) = b_{j} ,\;\; &&r_{\vb,i}(x_{j}, y^{(1)}) = 1, &&&j=i
        \end{aligned}
    \end{cases} ,\;\;
    \forall j\in\naturalnumber .
\end{equation*}
Define $\rwc=\{r_{\vb,i}, \forall \vb\in\{0,1\}^{\mathbb{N}}, \forall i\in\naturalnumber\}$.
\end{example}

As an extension of \Cref{ex:not-vc}, \Cref{ex:pessimism-fails} provides a reward class with an infinite VC dimension and an alignment dimension $d_{\mathrm{ALN}}(\rwc)=1$. 
Given a finite training sample, if the test instance does not show up in the training data and the candidate pool contains both $y^{(0)}$ and $y^{(1)}$, $\hat{\rwc}_{m}$ contains reward functions that realize both $(1,0)$ and $(0,1)$ patterns of rewards for $((x,y^{(0)}),(x,y^{(1)}))$.
Therefore, the pessimism principle yields a reward function that assigns reward value $0$ to both $(x,y^{(0)})$ and $(x,y^{(1)})$, and thus cannot return a good response at test time.

Given the failure of the ERM and pessimism principles, tie-breaking (at the test point) seems to matter a lot, guiding us towards considering the one-inclusion graph algorithm.

\section{Proof of Theorem~\ref{thm:alignment-learnability}}
  \label{sec:proof-learnability}
In this section, we prove the learnability characterization stated in \Cref{thm:alignment-learnability}. The main idea is to reduce the alignment problem to learning a partial binary comparison class. The learner first learns to distinguish a good response from a bad response when they are presented as a pair. At test time, it compares every pair of candidate responses and picks a most-winning candidate. We prove the upper bound in \Cref{subsec:pairwise-comparison}-\Cref{subsec:upper-bound}. We then prove the matching necessity result in \Cref{subsec:lower-bound} by constructing a hard family of independent pairwise preferences.

\subsection{Learning the pairwise comparison class}
  \label{subsec:pairwise-comparison}
For any reward function $r\in\rwc$, define its induced partial pairwise comparison function $g_{r}:\mathcal{X}\times\mathcal{Y}\times\mathcal{Y}\to\{0,1,*\}$ according to
\begin{equation*}
  g_{r}(x,y^{(0)},y^{(1)})=
  \begin{cases}
    0, & r(x,y^{(0)})=1\;\;\text{and}\;\;r(x,y^{(1)})=0,\\
    1, & r(x,y^{(0)})=0\;\;\text{and}\;\;r(x,y^{(1)})=1,\\
    *, & r(x,y^{(0)})=r(x,y^{(1)}).
  \end{cases}
\end{equation*}
Let $\mathcal{G}(\rwc)=\{g_{r}:r\in\rwc\}$ be the induced partial comparison class. It is easy to notice that the VC dimension of the partial class $\mathcal{G}(\rwc)$ is equivalent to the alignment dimension of the reward class $\rwc$, $d_{\mathrm{ALN}}(\rwc)$.

Let $\mathcal{U}=\mathcal{X}\times\mathcal{Y}\times\mathcal{Y}$. For a fixed target reward function $r^{*}\in\rwc$, define $g^{*}=g_{r^{*}}$. Let $\mathcal{D}_{\mathrm{pair}}$ be the distribution over $\mathcal{U}$ generated according to
\begin{equation}
  \label{eq:pair-distribution}
  X\sim\mathcal{D}_{X},\qquad Y^{(0)},Y^{(1)}\overset{\mathrm{i.i.d.}}{\longsim}\piref(\cdot|X),\qquad U=(X,Y^{(0)},Y^{(1)}).
\end{equation}
The comparison label $g^{*}(U)$ is observable from the binary reward labels: it equals $0$ when the first response is good and the second is bad, equals $1$ in the reverse case, and equals $*$ when the two responses have the same reward value.

We first recall a standard consequence of the PAC learnability of partial concept classes. We state it in a form that controls the error under the ambient distribution $\mathcal{D}_{\mathrm{pair}}$, rather than only under the distribution conditioned on observing a non-$*$ label.

\begin{lemma}
  \label{lem:partial-comparator-learning}
Let $\mathcal{H}\subseteq\{0,1,*\}^{\mathcal{U}}$ be a partial concept class with $\mathrm{VC}(\mathcal{H})=d<\infty$, and let $h^{*}\in\mathcal{H}$. Suppose that $U_{1},\ldots,U_{m}$ are generated i.i.d. from a distribution $\mathcal{D}$ over $\mathcal{U}$ and that the partial labels $h^{*}(U_{1}),\ldots,h^{*}(U_{m})$ are observed. There exists a learning algorithm $\mathbb{A}_{\mathrm{OIG}}$ which returns a total predictor $\hat{h}_{m}:\mathcal{U}\to\{0,1\}$ such that, for any $\eta,\delta\in(0,1)$,
\begin{equation}
  \label{eq:partial-comparator-sample-complexity}
  m \geq C\frac{1}{\eta}\left((d\vee 1)\log^{2}\left(\frac{d\vee 1}{\eta}\right)+\log\left(\frac{1}{\delta}\right)\right)
\end{equation}
for a universal constant $C>0$ implies
\begin{equation}
  \label{eq:partial-comparator-error}
  \Pr_{U\sim\mathcal{D}}\left(h^{*}(U)\in\{0,1\}\;\;\mathrm{and}\;\;\hat{h}_{m}(U)\neq h^{*}(U)\right) \leq \eta
\end{equation}
with probability at least $1-\delta$ over the training sample and the internal randomness of the learner. Here, $d\vee 1=\max\{d,1\}$.
\end{lemma}

\begin{proof}
Let
\begin{equation*}
  \mu=\Pr_{U\sim\mathcal{D}}\left(h^{*}(U)\in\{0,1\}\right)
\end{equation*}
be the probability of observing a non-$*$ label. Discard all examples whose labels equal $*$. Conditional on the number $N$ of retained examples, the retained examples are generated i.i.d. from the distribution $\mathcal{D}$ conditioned on $h^{*}(U)\in\{0,1\}$, and they are realizable by the partial class $\mathcal{H}$.

We run the one-inclusion graph learner followed by the boosting and sample-compression procedure of \cite{alon2022partial} on these $N$ retained examples. Conditional on $N=n\geq 1$, its conditional error is at most
\begin{equation}
  \label{eq:oig-conditional-error}
  C_{0}\frac{(d\vee 1)\log^{2}((d\vee 1)n)+\log(2/\delta)}{n}
\end{equation}
with probability at least $1-\delta/2$, where $C_{0}>0$ is a universal constant. When $N=0$, the learner returns an arbitrary binary predictor.

If $\mu\leq\eta$, then the left-hand side of \Cref{eq:partial-comparator-error} is at most $\mu\leq\eta$ for every predictor. Suppose next that $\mu>\eta$. Since $N\sim\mathrm{Bin}(m,\mu)$, a multiplicative Chernoff bound gives
\begin{equation*}
  \Pr\left(N<\frac{m\mu}{2}\right)\leq \exp\left(-\frac{m\mu}{8}\right)\leq\frac{\delta}{2},
\end{equation*}
where the last inequality follows from \Cref{eq:partial-comparator-sample-complexity} after increasing the universal constant $C$ if necessary. On the event $N\geq m\mu/2$, multiplying the conditional error in \Cref{eq:oig-conditional-error} by $\mu$ shows that the error under the ambient distribution is at most
\begin{equation*}
  C_{0}\mu\frac{(d\vee 1)\log^{2}((d\vee 1)m)+\log(2/\delta)}{N}
  \leq
  \frac{2C_{0}}{m}\left((d\vee 1)\log^{2}((d\vee 1)m)+\log(2/\delta)\right).
\end{equation*}
The choice of $m$ in \Cref{eq:partial-comparator-sample-complexity}, with a sufficiently large universal constant, makes the last display at most $\eta$. Taking a union bound over the Chernoff event and the failure event of the partial-class learner proves the lemma.
\end{proof}

We apply \Cref{lem:partial-comparator-learning} to $\mathcal{H}=\mathcal{G}(\rwc)$. A pairwise training sample can be extracted from the original slate-valued sample $S_{m}$. For every $i\in[m]$, independently choose an ordered pair of distinct indices $(P_{i},Q_{i})$ uniformly from $[k]\times[k]$. Define
\begin{equation*}
  U_{i}=\left(X_{i},Y_{i}^{(P_{i})},Y_{i}^{(Q_{i})}\right).
\end{equation*}
Because the responses within every slate are i.i.d., $U_{1},\ldots,U_{m}$ are i.i.d.\ according to $\mathcal{D}_{\mathrm{pair}}$. Moreover, the label $g^{*}(U_{i})$ is known from $A_{r^{*}}(Z_{i})$: it is $0$ if $P_{i}\in A_{r^{*}}(Z_{i})$ and $Q_{i}\notin A_{r^{*}}(Z_{i})$, it is $1$ in the reverse case, and it is $*$ otherwise.

\subsection{The pairwise tournament learning procedure}
  \label{subsec:pairwise-tournament}
We now present the learning procedure. The training stage learns the partial pairwise comparison class using \Cref{lem:partial-comparator-learning}. At test time, the learner compares every pair of candidate responses. Each pairwise winner receives one point, and the learner returns a response with the largest number of wins. Our learning strategy is formally presented in \Cref{fig:pairwise-tournament-algorithm}.

\begin{figure}[t]
\fbox{
  \parbox{0.97\linewidth}{
    \begin{itemize}[topsep=0pt,leftmargin=0pt]
      \item[] \textbf{Training:}
      \begin{itemize}
        \item For every $i\in[m]$, independently choose an ordered pair of distinct indices $(P_i,Q_i)$ uniformly from $[k]\times[k]$.
        \item Set $U_i=(X_i,Y_i^{(P_i)},Y_i^{(Q_i)})$ and determine the partial label $g^{*}(U_i)\in\{0,1,*\}$ from $A_{r^{*}}(Z_i)$.
        \item Discard the examples with label $*$ and run the partial-class OIG learner to obtain a pairwise predictor $\hat{g}_{m}:\mathcal{U}\to\{0,1\}$.
      \end{itemize}
      \item[] \textbf{Inference:}
      \begin{itemize}
        \item Let $Z=(X,Y^{(1)},\ldots,Y^{(k)})$ be a test instance and initialize $W_{j}=0$ for every $j\in[k]$.
        \item For every $1\leq a<b\leq k$, compute $\hat{g}_{m}(X,Y^{(a)},Y^{(b)})$.
        \begin{itemize}
          \item If the prediction is $0$, update $W_{a}\leftarrow W_{a}+1$.
          \item If the prediction is $1$, update $W_{b}\leftarrow W_{b}+1$.
        \end{itemize}
        \item Return $Y^{(\hat{j})}$, where $\hat{j}$ is chosen uniformly from $\argmax_{j\in[k]}W_{j}$.
      \end{itemize}
    \end{itemize}
  }
}
\caption{The pairwise tournament learning procedure used in \Cref{thm:upper-bound}.}
\label{fig:pairwise-tournament-algorithm}
\end{figure}

The next elementary lemma explains why the tournament rule converts accurate mixed-pair comparisons into a correct selection.

\begin{lemma}
  \label{lem:tournament-selection}
Fix a test slate $Z=(X,Y^{(1)},\ldots,Y^{(k)})$ satisfying $A_{r^{*}}(Z)\neq\emptyset$. Suppose that for every pair containing exactly one accepted response, $\hat{g}_{m}$ selects correctly. Then every maximizer of the score $W_{j}$ belongs to $A_{r^{*}}(Z)$.
\end{lemma}

\begin{proof}
Let $G=|A_{r^{*}}(Z)|$ and $B=k-G$. Every accepted response defeats all $B$ rejected responses, and therefore has score at least $B$. On the other hand, every rejected response loses to all $G$ accepted responses and can defeat at most the other $B-1$ rejected responses. Hence, every rejected response has score at most $B-1$. It follows that every maximizer of the score is accepted.
\end{proof}

\subsection{Upper bound}
  \label{subsec:upper-bound}
We are now ready to prove the sufficiency direction of \Cref{thm:alignment-learnability}.

\begin{theorem}
  \label{thm:upper-bound}
Suppose that $d_{\mathrm{ALN}}(\rwc)=d<\infty$. For any $\gamma,\epsilon,\delta\in(0,1)$, choose
\begin{equation}
  \label{eq:k-upper-bound}
  k\geq \max\left\{2,\left\lceil\frac{1}{\gamma}\log\left(\frac{2}{\epsilon}\right)\right\rceil\right\}.
\end{equation}
There exists a universal constant $C>0$ such that the pairwise tournament procedure satisfies
\begin{equation*}
  \Pr_{Z\sim\mathcal{D}_{k}}\left(\hat{i}_{m}(Z)\in A_{r^{*}}(Z)\right)\geq 1-\epsilon
\end{equation*}
with probability at least $1-\delta$, provided that
\begin{equation}
  \label{eq:upper-bound-sample-complexity}
  m\geq C\frac{k^{2}}{\epsilon}\left((d\vee1)\log^{2}\left(\frac{(d\vee1)k}{\epsilon}\right)+\log\left(\frac{1}{\delta}\right)\right).
\end{equation}
In particular,
\begin{equation*}
  m=\widetilde{O}\left(\frac{k^{2}}{\epsilon}\left(d+\log\left(\frac{1}{\delta}\right)\right)\right)
\end{equation*}
suffices.\footnote{When using $\tilde{O}(\cdot)$ notation, we save polylogarithmic factors in $d$, $k$ and $1/\epsilon$.}
\end{theorem}

\begin{proof}
Set
\begin{equation*}
  \eta=\frac{\epsilon}{2\binom{k}{2}}.
\end{equation*}
By \Cref{def:alignment-dimension}, the partial comparison class $\mathcal{G}(\rwc)$ has VC dimension $d$. Hence, \Cref{lem:partial-comparator-learning} and \Cref{eq:upper-bound-sample-complexity} imply that, with probability at least $1-\delta$ over the training sample,
\begin{equation}
  \label{eq:mixed-pair-error-bound}
  \Pr_{U\sim\mathcal{D}_{\mathrm{pair}}}\left(g^{*}(U)\in\{0,1\}\;\;\mathrm{and}\;\;\hat{g}_{m}(U)\neq g^{*}(U)\right)\leq\eta.
\end{equation}
Condition on a training sample for which \Cref{eq:mixed-pair-error-bound} holds. For a fresh test slate, every ordered pair $(X,Y^{(a)},Y^{(b)})$, with $a<b$, has marginal distribution $\mathcal{D}_{\mathrm{pair}}$. A union bound over the $\binom{k}{2}$ pairs yields
\begin{equation}
  \label{eq:test-pair-union-bound}
  \Pr_{Z\sim\mathcal{D}_{k}}\left(\text{some pair containing exactly one accepted response selects incorrectly}\right)
  \leq \binom{k}{2}\eta
  =\frac{\epsilon}{2}.
\end{equation}

Moreover, the coverage assumption and the conditional independence of the test responses imply
\begin{equation}
  \label{eq:no-good-response-probability}
  \Pr_{Z\sim\mathcal{D}_{k}}\left(A_{r^{*}}(Z)=\emptyset\right)
  \leq (1-\gamma)^{k}
  \leq e^{-\gamma k}
  \leq\frac{\epsilon}{2},
\end{equation}
where the last inequality follows from \Cref{eq:k-upper-bound}. Whenever the test slate contains an accepted response and no mixed pair is classified incorrectly, \Cref{lem:tournament-selection} guarantees that the returned response is accepted. Combining \Cref{eq:test-pair-union-bound} and \Cref{eq:no-good-response-probability} gives
\begin{equation*}
  \Pr_{Z\sim\mathcal{D}_{k}}\left(\hat{i}_{m}(Z)\notin A_{r^{*}}(Z)\right)\leq\epsilon.
\end{equation*}
This proves the result.
\end{proof}

\subsection{Lower bound}
  \label{subsec:lower-bound}
We next prove that the finiteness of the alignment dimension is necessary. The lower bound uses the fact that a shattered family contains many independent pairwise preferences, including preferences associated with different pairs sharing the same prompt.

\begin{lemma}
  \label{lem:disjoint-pairs-same-prompt}
Suppose that $u_{1},\ldots,u_{d}$ are alignment shattered by $\rwc$, where $u_{i}=(x_{i},y_{i}^{(0)},y_{i}^{(1)})$. If $i\neq j$ and $x_{i}=x_{j}$, then
\begin{equation*}
  \{y_{i}^{(0)},y_{i}^{(1)}\}\cap\{y_{j}^{(0)},y_{j}^{(1)}\}=\emptyset.
\end{equation*}
\end{lemma}

\begin{proof}
Suppose for contradiction that $y_{i}^{(a)}=y_{j}^{(c)}$ for some $a,c\in\{0,1\}$. By alignment shattering, there exists a reward function corresponding to a boolean vector satisfying $b_{i}=a$ and $b_{j}=1-c$. The shattering conditions would require the shared response to have reward $1$ through the $i$-th pair and reward $0$ through the $j$-th pair, which is impossible.
\end{proof}

The following lemma gives a finite-dimensional form of the lower bound. It is sufficient for the necessity direction and also records an explicit dependence on the slate size $k$.

\begin{lemma}
  \label{lem:finite-dimensional-lower-bound}
Suppose that $d_{\mathrm{ALN}}(\rwc)\geq d$. Fix any learning algorithm, any training sample size $m\in\naturalnumber$, and any slate size $k\in\naturalnumber$. If
\begin{equation*}
    d\geq 4mk^{2},
\end{equation*}
there exist a target reward function $r^{*}\in\rwc$, a distribution $\mathcal{D}_{X}$, and a reference policy $\piref$ satisfying the coverage condition with $\gamma=1/2$, such that
\begin{equation}
  \label{eq:lower-bound-high-probability}
  \Pr_{S_{m}}\left(\Pr_{Z\sim\mathcal{D}_{k}}\left(\hat{i}_{m}(Z)\notin A_{r^{*}}(Z)\right)\geq\frac{1}{8}\right)\geq\frac{2}{7}.
\end{equation}
The probability also includes the internal randomness of the learner when the learner is randomized.
\end{lemma}

\begin{proof}
Let
\begin{equation*}
  u_{i}=\left(x_{i},y_{i}^{(0)},y_{i}^{(1)}\right),\qquad i\in[d],
\end{equation*}
be alignment shattered by $\rwc$. For every prompt $x$ appearing among $x_{1},\ldots,x_{d}$, define
\begin{equation*}
  I_{x}=\{i\in[d]:x_{i}=x\},\qquad n_{x}=|I_{x}|.
\end{equation*}
By \Cref{lem:disjoint-pairs-same-prompt}, the response pairs indexed by $I_x$ are disjoint. Define
\begin{equation}
  \label{eq:lower-bound-prompt-distribution}
  \mathcal{D}_{X}(x)=\frac{n_{x}}{d}
\end{equation}
for every prompt appearing in the shattered family. Conditional on such a prompt $x$, define the reference policy by
\begin{equation}
  \label{eq:lower-bound-reference-policy}
  \piref\left(y_{i}^{(0)}\mid x\right)=\piref\left(y_{i}^{(1)}\mid x\right)=\frac{1}{2n_{x}},\qquad i\in I_{x}.
\end{equation}

Choose a boolean vector $\vb$ uniformly from $\{0,1\}^{d}$, and let $r_{\vb}\in\rwc$ be a reward function witnessing the corresponding preferences. For every prompt in the support of $\mathcal{D}_{X}$, exactly one endpoint of each pair has reward $1$. Consequently,
\begin{equation*}
  \Pr_{Y\sim\piref(\cdot|x)}\left(r_{\vb}(x,Y)=1\right)=\frac{1}{2},
\end{equation*}
so the coverage condition holds with $\gamma=1/2$.

For a realized training sample $S_m$, let $J(S_m)\subseteq[d]$ denote the set of pair indices for which at least one endpoint appears among the $mk$ training responses. Clearly,
\begin{equation*}
  |J(S_m)|\leq mk.
\end{equation*}
For every prompt $x$, let $s_x=|J(S_m)\cap I_x|$. Conditional on $X=x$, every test response first chooses a pair index uniformly from $I_x$ and then chooses one of its two endpoints uniformly. Hence, conditional on $S_m$, the probability that all $k$ test responses come from pair indices outside $J(S_m)$ is
\begin{align*}
  \sum_{x}\frac{n_x}{d}\left(1-\frac{s_x}{n_x}\right)^k \geq{}& \sum_x\frac{n_x}{d}\left(1-\frac{ks_x}{n_x}\right) \\
  ={}& 1-\frac{k}{d}\sum_x s_x \\
  ={}& 1-\frac{k|J(S_m)|}{d} \\
  \geq{}& 1-\frac{mk^2}{d} \\
  \geq{}& \frac{3}{4},
\end{align*}
where the first inequality uses $(1-t)^k\geq1-kt$ for $t\in[0,1]$.

On this event, whichever candidate the learner selects belongs to a pair whose preference bit $b_i$ is not revealed by the training sample. Under the uniformly random choice of $\vb$, that preference bit is independent of the training sample, the test slate, and the learner's selection rule. Therefore, the selected response has reward $0$ with probability $1/2$. It follows that
\begin{equation*}
  \E_{\vb,S_m}\left[\Pr_{Z\sim\mathcal{D}_{k}}\left(\hat{i}_{m}(Z)\notin A_{r_{\vb}}(Z)\right)\right]\geq\frac{1}{2}\cdot\frac{3}{4}=\frac{3}{8}.
\end{equation*}
By the probabilistic method, there exists a deterministic $\vb$ such that, for $r^{*}=r_{\vb}$,
\begin{equation*}
  \E_{S_m}\left[\Pr_{Z\sim\mathcal{D}_{k}}\left(\hat{i}_{m}(Z)\notin A_{r^{*}}(Z)\right)\right]\geq\frac{3}{8}.
\end{equation*}
Let
\begin{equation*}
  R(S_m)=\Pr_{Z\sim\mathcal{D}_{k}}\left(\hat{i}_{m}(Z)\notin A_{r^{*}}(Z)\right)\in[0,1].
\end{equation*}
Since
\begin{equation*}
  \E[R(S_m)]\leq\frac{1}{8}\Pr\left(R(S_m)<\frac{1}{8}\right)+\Pr\left(R(S_m)\geq\frac{1}{8}\right),
\end{equation*}
we obtain
\begin{equation*}
  \Pr\left(R(S_m)\geq\frac{1}{8}\right)\geq\frac{\frac{3}{8}-\frac{1}{8}}{1-\frac{1}{8}}=\frac{2}{7}.
\end{equation*}
This proves \Cref{eq:lower-bound-high-probability}.
\end{proof}

\begin{theorem}
  \label{thm:lower-bound}
If $d_{\mathrm{ALN}}(\rwc)=\infty$, then $\rwc$ is not alignment learnable.
\end{theorem}

\begin{proof}
Fix an arbitrary learning algorithm and set
\begin{equation*}
  \gamma=\frac{1}{2},\qquad \epsilon=\frac{1}{16},\qquad \delta=\frac{1}{4}.
\end{equation*}
For any finite slate size $k$ and any finite training sample size $m$, the assumption $d_{\mathrm{ALN}}(\rwc)=\infty$ allows us to choose $d\geq4mk^{2}$ alignment-shattered response pairs. Applying \Cref{lem:finite-dimensional-lower-bound}, there exist a target reward, a prompt distribution, and a reference policy satisfying the $\gamma$-coverage condition for which the learned selector has test error at least $\epsilon$ with probability at least $2/7>\delta$. This contradicts \Cref{def:alignment-learnability}. Therefore, $\rwc$ is not alignment learnable.
\end{proof}

Combining \Cref{thm:upper-bound} and \Cref{thm:lower-bound} proves \Cref{thm:alignment-learnability}.

\section{Discussion}
  \label{sec:discussion}
We formulate the problem of inference-time alignment into a statistical learning framework, where the goal is to boost the prediction performance of a reference policy by leveraging the preference data.
We define the notion of \emph{alignment learnability} in the sense of weak-to-strong learning: given a reference policy (weak model), is there a finite sample size such that i.i.d.\ training data can be used to produce a strong model that predicts good responses with arbitrarily high probability at test time?
We seek a characterization of alignment learnability --- asking which reward classes are alignment learnable.
As an initial work, we restrict to the realizable case as well as binary rewards.
We introduce a novel and natural combinatorial dimension of the reward class called the alignment dimension, which provides a complete characterization of the alignment learnability --- a reward class is alignment learnable if and only if its alignment dimension is finite.

We consider our work as an initial step towards establishing a general statistical learning theory for alignment. 
Since we just figured out its learnability, a natural follow-up question is to quantify the optimal sample complexity required to achieve the $(\gamma,\epsilon,\delta)$ alignment learnability, namely, the optimal dependencies in terms of $\gamma$, $\epsilon$ and $\delta$.
Moreover, we notice that our learning model differs from existing frameworks such as multiclass learning and list learning, and itself might be of independent interest to the learning theory community.
Indeed, various extensions such as online learning and active learning could be interesting to study. 
The following problems are those more closely related to our formulation.

\subsection{Learning with single-response training data}
  \label{subsec:learning-with-single-response-training-data}
One of the main constraints of our theory is that we assume having multiple responses $y^{(1)},\ldots,y^{(k)}$ sampled from the reference policy $\piref(\cdot|x)$ for each input $x$. 
This assumption arises from a technical reason.
Since the reference policy has a constant probability of outputting a good response, having multiple responses sampled aims to ensure that there exists a good response within the generated tuple.
In particular, having multiple responses at test time is necessary, since the learner is expected to select a response from the test slate. 
This leads to its necessity also in the training data generation since our learning procedure invokes the leave-one-out one-inclusion graph algorithm, which relies on the exchangeability of training and test instances.

On the other hand, in practical language model alignment tasks, usually only two responses are collected for each prompt to be compared by human annotators. 
This motivates the following open problem: which reward classes are alignment-learnable with single-response training data (and of course we still collect multiple responses at test time).
Note that for binary rewards, having one or two responses makes almost no difference, since in both cases the event that all sampled responses have reward values $0$ or $1$ happens with constant probability. 
We conjecture that our alignment dimension might also characterize the ``single-response'' alignment learnability.

\subsection{Learning to predict rather than to select}
  \label{subsec:learning-to-predict-rather-than-to-select}
Another constraint of our formulation is that the learner is expected to select a good response from multiple candidates. 
However, practical alignment methods, especially RLHF methods, instead seek to predict a good response for a given test prompt.\footnote{Note that this does not necessarily require learning a policy from a hypothesis policy class.} 
More formally, under the setting of \Cref{sec:preliminaries}, the goal is then to design a learning algorithm $\mathcal{A}$ that takes $S_{m}$ as input and outputs a predictor $\mathcal{A}(S_{m}):\mathcal{X}\rightarrow\mathcal{Y}$ that, given a test instance $X\sim\mathcal{D}_{X}$, predicts a response $\hat{Y}=\mathcal{A}(S_{m})(X)\in\mathcal{Y}$ such that $r^{*}(X,\hat{Y})=1$ with high probability.

While it might be interesting to have a theoretical understanding of this learning model, it is worth mentioning that this formulation mismatches practical alignment in a sense that, the learner's prediction is irrelevant to the reference policy $\piref$.
However, for practical alignment such as RLHF approaches, a regularization of KL-divergence $\mathbb{D}_{\text{KL}}(\cdot || \piref)$ is usually appended to a policy optimization objective to ensure that the learned policy (predictor) stays close to the reference policy and thus may inherit nice properties from it such as linguistic fluency and broad knowledge.
It is also for this reason that we study selection-based prediction in this work.

\subsection{Learning with preferences}
  \label{subsec:learning-with-preference}
A more significant and challenging extension is to study exactly alignment with preferences. 
Specifically, consider the setting of real-valued rewards, for example, $\rwc\subseteq[0,1]^{\mathcal{X}\times\mathcal{Y}}$.
Every training sample consists of a prompt $x\in\mathcal{X}$, two independent responses $y^{(0)},y^{(1)}$ generated from the reference policy $\piref(\cdot|x)$, and a binary preference label $\sigma=\mathbbm{1}[y^{(0)}\succ y^{(1)}]$ indicating which response is more preferred according to the true reward function $r^{*}\in\rwc$. 
This setting could be considered as an offline variant of dueling bandits.
Given i.i.d.\ samples $S_{m}=\{(x_{i},y_{i}^{(0)},y_{i}^{(1)},\sigma_{i})\}_{i=1}^{m}$, the learner $\mathcal{A}$ is expected to predict a response $\mathcal{A}(S_{m})(x)\in\mathcal{Y}$ on a test prompt $x$ such that
\begin{equation*}
    \Pr_{x\sim\mathcal{D}_{X}, y'\sim\piref(\cdot|x)}\left(\mathcal{A}(S_{m})(x) \succ y'\right) \geq 1-\epsilon .
\end{equation*}
with probability at least $1-\delta$.

The challenge arises in how to measure the quality of the reference policy. 
\cite{huang2025bestofn} made an assumption on the reference policy's coverage. 
They defined the coverage coefficient as
\begin{equation*}
    \mathcal{C}^{\pi^{*}}(x) = \E_{y\sim\pi^{*}(\cdot|x)}\left[\frac{\pi^{*}(y|x)}{\piref(y|x)}\right] 
\end{equation*}
where $\pi^{*}$ is a comparator policy that achieves high true rewards, and assumed that $\mathcal{C}^{\pi^{*}}(x)\leq C$ for some constant $C>0$ holds uniformly for any $x$.
\cite{sriraman2026revisiting} proposed another measure of discrepancy between $\piref$ and $\pi^{*}$ called the $\mathcal{E}_{M}$-divergence. 
\cite{zhao2025sharp} studied other coverage conditions including the Global-Policy Coverage and the Local KL-ball Coverage. 
Note that such assumptions are usually formulated as discrepancy measures between $\piref$ and a comparator policy. 
Instead, can we model the reference policy assumption to only rely on parameters $\gamma$, $\epsilon$ and $\delta$?

\section*{AI disclosure}
The first version of this work was developed by the authors in August 2026. Later, the authors used OpenAI GPT-5.6 Sol in the Extra High effort mode to make corrections, improve bounds, simplify proofs, and to help with writing.


\bibliographystyle{alphaurl}
\bibliography{rlvr}

\end{document}